\documentclass[10pt,journal,compsoc]{IEEEtran}

\usepackage{lipsum}

\usepackage{amsthm}
\theoremstyle{definition}
\newtheorem{definition}{Definition}[section]
\newtheorem{theorem}{Theorem}[section]

\usepackage{pgfplots} 
\usetikzlibrary{patterns}
\usepackage{tikz}
\usepgfplotslibrary{groupplots}
\pgfplotsset{compat=1.16}

\usepackage{diagbox}
\usepackage{array}
\usepackage{graphicx}
\usepackage{multirow}
\newcommand\MyBox[2]{
  \fbox{\lower0.75cm
    \vbox to 2cm{\vfil
      \hbox to 2cm{\hfil\parbox{1.5cm}{#1\\#2}\hfil}
      \vfil}%
  }%
}

\newcommand\MySBox[1]{
  \fbox{\lower0.75cm
    \vbox to 1cm{\vfil
      \hbox to 1cm{\hfil\parbox{.5cm}{#1}\hfil}
      \vfil}%
  }%
}

\newcommand\MySCBox[1]{
  \fbox{\lower0.75cm
    \vbox to 1cm{\vfil
      \hbox to 1cm{\hfil\parbox{.5cm}{#1}\hfil}
      \vfil}%
  }%
}

\newcommand\MySCBoxB[1]{
  \fbox{\lower0.75cm
    \vbox to 1cm{\vfil
      \hbox to 1cm{\parbox{.5cm}{#1}\hfil}
      \vfil}%
  }%
}

\usepackage{amsmath}
\usepackage[ruled, vlined, linesnumbered]{algorithm2e}
\usepackage{caption}
\usepackage{subcaption}

\ifCLASSOPTIONcompsoc
  \usepackage[nocompress]{cite}
\else
  \usepackage{cite}
\fi
\ifCLASSINFOpdf
\else
\fi
\begin{document}

%
\title{Prayatul Matrix: A Direct Comparison Approach to Evaluate Performance of Supervised \\Machine Learning Models}
%
%
%
%

\author{Anupam~Biswas,~\IEEEmembership{Member,~IEEE}
\IEEEcompsocitemizethanks{\IEEEcompsocthanksitem A. Biswas was with the Department of Computer Science and Engineering, National Institute of Technology Silchar, Assam-788010, India.\protect\\
E-mail: anupam@cse.nits.ac.in }
\thanks{Manuscript received Month xx, yyyy; revised Month xx, yyyy.}}

%
%

\markboth{Journal of \LaTeX\ Class Files,~Vol.~xx, No.~xx, Month~yyyy}%
{Shell \MakeLowercase{\textit{et al.}}: Bare Advanced Demo of IEEEtran.cls for IEEE Computer Society Journals}
%



\IEEEtitleabstractindextext{%
\begin{abstract}
Performance comparison of supervised machine learning (ML) models are widely done in terms of different confusion matrix based scores obtained on test datasets. However, a dataset comprises several instances having different difficulty levels. Therefore, it is more logical to compare effectiveness of ML models on individual instances instead of comparing scores obtained for the entire dataset. In this paper, an alternative approach is proposed for direct comparison of supervised ML models in terms of individual instances within the dataset. A direct comparison matrix called \emph{Prayatul Matrix} is introduced, which accounts for comparative outcome of two ML algorithms on different instances of a dataset. Five different performance measures are designed based on prayatul matrix. Efficacy of the proposed approach as well as  designed measures is analyzed with four classification techniques on three datasets. Also analyzed on four large-scale complex image datasets with four deep learning models namely ResNet50V2, MobileNetV2, EfficientNet, and XceptionNet. Results are evident that the newly designed measure are capable of giving more insight about the comparing ML algorithms, which were impossible with existing confusion matrix based scores like accuracy, precision and recall.
\end{abstract}

}

\maketitle

\IEEEdisplaynontitleabstractindextext

%
\IEEEpeerreviewmaketitle

\ifCLASSOPTIONcompsoc
\IEEEraisesectionheading{\section{Introduction}\label{sec:introduction}}
\else
\section{Introduction}
\label{sec:introduction}
\fi

 \IEEEPARstart{P}{erformance} evaluation of supervised ML Models in general is done on the basis of confusion matrix, irrespective of application domains~\cite{sokolova2009systematic,jordan2015machine,Pouyanfar2018,fatima2017survey,qiu2016survey}.  A confusion matrix is a kind of contingency table, where each row represents an actual class, while each column represents a predicted class. Be it classification~\cite{Zhang2017} or clustering algorithm~\cite{xu2005survey}, confusion matrix is prepared to evaluate and visually describe the performance of the model on a test dataset for which the ground truth values are known. The well-known measures like accuracy, precision, and recall etc. are computed based on confusion matrix to enumerate the performance of ML models. The confusion matrix or the measures that are computed based on confusion matrix are single model driven. To compare performance of two ML models, two separate confusion matrices have to be prepared for each model and relevant measures have to be computed on the basis of these matrices. One of the major drawbacks of such performance comparison is it lacks direct comparison of ML models on individual instances of the dataset. For instance, one can compare true positive value of one confusion matrix with another, which  determines how many times two models are correct. However, it cannot determine exactly on which instances the models are correct or whether the models are correct on the same or different instances. Important to note that both the models can have the same true positive values but instance-wise those may be completely opposite. Therefore, the same true positive values or accuracy or any other measures based on confusion matrix do not mean that performance of both ML models will be same.

 Though, different measures are available that are not based on confusion matrix~\cite{sokolova2009systematic,Ferri2009}, mostly share the same drawback as mentioned above. In this paper, an alternative approach is proposed that enables direct comparison of ML models at the level of instances within the dataset. A direct comparison matrix called \emph{Prayatul Matrix} is prepared for accounting as well as visualizing the comparative outcome of two ML algorithms. The key features of the proposed approach are as follows:
\begin{itemize}
    \item Comparative outcome of two ML models is presented in a direct comparison matrix called \emph{prayatul matrix}.
    \item Instance level outcomes of both the ML models are compared in reference to the ground truth of the dataset.
    \item Five performance measures are defined based on the elements of prayatul matrix, which indicate a direct comparative scores between two ML models.
    \item All five measures satisfy important properties such as scale invariance, data invariance, monotonicity and continuity.
\end{itemize}

Rest of the paper is organized as follows. Section~\ref{sec:method} elaborates the proposed direct comparison approach for evaluating supervised ML algorithms, the prayatul matrix and measures designed. Section~\ref{sec:expAnalysis} details about experimental analysis covering experimental setup, datasets, and result analysis. Section~\ref{sec:relwork} discusses the works related to contingency matrix based measures. Section~\ref{sec:conclusion} concludes highlighting the key advantages of the proposed approach.

\begin{figure}
    \centering
 \noindent
\renewcommand\arraystretch{1.5}
\setlength\tabcolsep{0pt}
\begin{tabular}{c >{\bfseries}r @{\hspace{0.7em}}c @{\hspace{0.4em}}c @{\hspace{0.7em}}l}
   & 
    & \multicolumn{2}{c}{\bfseries Alternative} & \\
  & & \bfseries Right& \bfseries Wrong & \bfseries Total \\
  \multirow{4}{*}{\rotatebox{90}{\parbox{2.9cm}{\bfseries \centering Primary}}}& \multirow{2}{*}{\rotatebox{90}{\parbox{0.7cm}{\bfseries\centering Right}}} & \MyBox{Both}{Right} & \MyBox{Right}{Wrong} & BR+RW \\[2.4em]
  & \multirow{2}{*}{\rotatebox{90}{\parbox{0.7cm}{\bfseries\centering Wrong}}} & \MyBox{Wrong}{Right} & \MyBox{Both}{Wrong} & WR+BW \\
  & Total & BR+WR & RW+BW & \bfseries N
\end{tabular}
    \caption{Prayatul Matrix for Primary and Alternative algorithms with different abstractions}
    \label{fig:compmat}
\end{figure}

\section{Direct Comparison Approach}
\label{sec:method}
The direct comparison approach involves two ML algorithms in the process. The role of the participating ML algorithms in the direct comparison are defined as follows:
\begin{definition}[Primary Algorithm ($A_p$)]
The algorithm whose performance is to be evaluated in comparison to other algorithm is referred as primary algorithm.
\end{definition}

\begin{definition}[Alternative Algorithm ($A_q$)]
The algorithms with whom the primary algorithm is to be compared is referred as alternative algorithm.
\end{definition}

The primary algorithm can be compared with multiple alternatives on same or different datasets. Let us consider a test dataset having $N$ instances with ground truths $G=\{g_1, g_2, g_3,..., g_N\}$. Let us consider $N$ outcomes obtained for $N$ different instances of test dataset with primary algorithm $A_p$ and alternative algorithm $A_q$ are $P=\{p_1, p_2, p_3,...p_N\}$ and $Q=\{q_1, q_2, q_3,...q_N\}$ respectively.

\subsection{Prayatul Matrix}
A $2\times 2$ dimensional direct comparison matrix $\mathcal{D}$ named \emph{Prayatul Matrix} is prepared by comparing the  outcomes of $A_p$ with that of $A_q$ individually for each instances w.r.t. ground truth $G$ of the test dataset. The prayatul matrix is a kind of contingency table that has two levels of abstractions  \emph{Right} and \emph{Wrong} both in rows and columns, which indicate the correctness of outcomes obtained with both primary and alternative algorithms. The abstraction \emph{Right} and \emph{Wrong} means an outcome of the algorithm is correct and incorrect respectively w.r.t. $G$.  Abstractions related to primary and alternative algorithms are placed in rows and columns respectively as shown in Fig.~\ref{fig:compmat}.

Let $R_p$ and $R_q$ respectively are the set of instances where prediction of $A_p$ and $A_q$  are correct i.e. the instances come under the abstraction \emph{Right}. Let $W_p$ and $W_q$ respectively are the set of instances where prediction of $A_p$ and $A_q$ are incorrect i.e. the instances come under the abstraction \emph{Wrong}. Now, the instances for which both $A_p$ and $A_q$ are \emph{Right} is given by $\{R_p\cap R_q\}$ and the corresponding entry for the matrix $\mathcal{D}$ is computed as follows:

\begin{equation}
    \mathcal{D}_{11}=|R_p\cap R_q|
\end{equation}

Likewise, the instances for which $A_p$ is \emph{Right} but $A_q$ is \emph{Wrong} is given by $\{R_p\cap W_q\}$ and the corresponding entry for the matrix $\mathcal{D}$ is computed as follows:
 
\begin{equation}
    \mathcal{D}_{12}=|R_p\cap W_q|
\end{equation}

The instances for which $A_p$ is \emph{Wrong} but $A_q$ is \emph{Right} is given by $\{W_p\cap R_q\}$ and the corresponding entry for the matrix $\mathcal{D}$ is computed as follows:
 
\begin{equation}
    \mathcal{D}_{21}=|W_p\cap R_q|
\end{equation}

Lastly, the instances for which both $A_p$ and $A_q$ are \emph{Wrong} is given by $\{W_p\cap W_q\}$ and the corresponding entry for the matrix $\mathcal{D}$ is computed as follows:

\begin{equation}
    \mathcal{D}_{22}=|W_p\cap W_q|
\end{equation}

Interpretation of different elements of the prayatul matrix in terms of abstractions \emph{Right} and \emph{Wrong} are done as follows:

\begin{itemize}
    \item \textbf{Both Right (BR):} Outcome of both $A_p$ and $A_q$ are same and both are right w.r.t. $G$.  
    \item \textbf{Right Wrong (RW):} Outcome of $A_p$ is right but outcome of $A_q$ is wrong w.r.t. $G$.
    \item \textbf{Wrong Right (WR):} Outcome of $A_p$  is wrong but outcome of $A_q$ is right w.r.t. $G$.
    \item \textbf{Both Wrong (BW):} Outcome of both $A_p$ and $A_q$ are same but both are wrong w.r.t. $G$.
\end{itemize}

\subsection{Comparative Performance Measures}
The elements of the prayatul matrix i.e. BR, RW, WR and BW are used to design five comparative performance measures for a pair of ML algorithms $A_p$ and $A_q$ as follows:

The elements RW and WR are the counts of instances, where both algorithms are having disagreement i.e. deviates from each others decision. If RW count is more that means primary algorithms is better in taking right decisions compared to alternative and it means opposite if WR is more.  Subtractions of WR from RW penalizes the wrong decisions of primary algorithm. Normalizing it with all deviating outcome counts i.e. RW+ WR gives the comparative deviation.  This measure indicates how two algorithms are deviating from each other when outcomes of both are different. Positive value implies primary algorithm is better, while negative value implies alternative algorithm is better in terms of right outcomes. Formally, the comparative deviation of $A_p$ and $A_q$ is defined as follows:

\begin{definition}[Comparative Deviation ($\sigma_c$)]
  The comparative deviation of primary algorithm over alternative algorithm is defined as:
\begin{equation}
    \sigma_c (P,Q)=\frac{RW-WR}{RW+WR}
\end{equation}
\end{definition}

On the other hand, the elements BR and BW are the counts of instances, where both algorithms agree. If BR  is high then it means both algorithms are polarized towards right decision, whereas it means opposite if BW is high. Subtractions of BR from BW penalizes the wrong decisions of both algorithms. Addition of RW with BR-BW gives the polarization of primary algorithm towards right decision in comparison to alternative algorithm.  Normalizing it with the total paired outcome counts gives the polarization of primary algorithm. This measure indicates how the primary algorithm is polarized towards right or wrong decision. Positive value implies primary algorithm is good at taking right decision and negative implies bad at taking right decision in comparison to alternative. The polarization of $A_p$ in comparison to $A_q$ is defined as follows:

\begin{definition}[Polarization ($\alpha$)]
The polarization of primary algorithm and alternative algorithm is defined as: 
\begin{equation}
    \alpha (P,Q)=\frac{BR+RW-BW}{BR + RW + WR + BW}
\end{equation}
\end{definition}

The elements BR and RW together gives the count of instances where primary algorithm is right. Normalizing it with the count of instances where at least one of the algorithms is right (i.e. BR + RW + WR) gives the comparative rightness of primary algorithm. While penalizing wrong decisions of primary gives the effective rightness of primary algorithm. Formally, comparative rightness and effective rightness of $A_p$ in comparison to $A_q$ is defined as follows:

\begin{definition}[Comparative Rightness ($\xi_c$)]
 The comparative rightness of primary algorithm over alternative algorithm is defined as:

\begin{equation}
    \xi_c (P,Q)=\frac{BR+RW}{BR+RW+WR}
\end{equation}
\end{definition}

\begin{definition}[Effective Rightness ($\xi_e$)]
The effective rightness of primary algorithm over alternative algorithm is defined as:
\begin{equation}
    \xi_e (P,Q)=\frac{BR+RW-WR}{BR+RW+WR}
\end{equation}
Higher $\xi_c$ and $\xi_e$ values indicate primary algorithm is good at taking right decisions and primary algorithm is good at taking right decisions despite of its wrong decisions respectively.
\end{definition}

Effective rightness measure indicates how good the primary algorithm is on taking right decision considering all right decisions and penalizing its wrong decision. However, it is from the perspective of all decisions where at least one of the algorithm is right. To have superiority over alternative, the primary algorithm has to perform better from the perspective of all decisions. Thus, effective superiority of primary algorithm in comparison to alternative is defined as follows:

\begin{definition}[Effective Superiority ($\phi_e$)]
 The effective superiority of primary algorithm over alternative algorithm is defined as:
\begin{equation}
    \phi_e (P,Q)=\frac{BR+RW-WR}{BR+RW+WR+BW}
\end{equation}
Higher $\phi_e$ value indicates primary algorithm is superior at taking right decisions in comparison to alternative algorithm.

\end{definition}

\begin{theorem}
\label{th:range}
The measures $\sigma_c, \alpha, \xi_e$ and $\phi_e$ have the range [-1, +1], but $\xi_c$ has the range [0, 1]. 
\end{theorem}

\begin{proof}
The proof is quite straight forward. Since all the measures except $\xi_c$ has a negative element in numerator, which can have maximum $N$ value and apparently other elements will be 0 in that case, implying minimum value -1. While both the elements in numerator of $\xi_c$ can have 0 simultaneously implying minimum value 0. For all measures, one positive element in numerator can have maximum $N$ value and apparently other elements will be 0 in that case, implying maximum value +1. 
\end{proof}
\subsection{Properties of Comparative Performance Measures}
Unlike the measures that are defined based on confusion matrix indicate the performance of a standalone algorithm, the five measures defined based on prayatul matrix indicate comparative performance of one algorithm over another. The properties of proposed comparative performance measures are analyzed theoretically in the context of following four axioms:
\begin{itemize}
    \item  \textbf{Scale Invariance:} Metric should scale irrespective of sample size small or large.
    \item  \textbf{Data Invariance:} Metric should not be affected by unbalances within the dataset. 
    \item  \textbf{Monotonicity:} Metric has to be non-decreasing under monotonic consistent improvement.
    \item  \textbf{Continuity:} Small change in $\#$samples $N$ has to cause a smaller impact on the metric. 
\end{itemize}

\begin{theorem}
\label{th:sinvariant}
    All five measures $\sigma_c, \alpha, \xi_c, \xi_e$ and $\phi_e$ are scale invariant.
\end{theorem}

\begin{proof}
All the proposed measures are multivariate functions only. Let us consider, the case of $\sigma_c$. By replacing RW and WR with variables $x_1$ and $x_2$ respectively, $\sigma_c$ can be written in the form of a function as follows:

\begin{equation}
    \sigma_c (x_1, x_2)= \frac{x_1-x_2}{x_1+x_2}.
\end{equation}

A multivariate function $f(x_1, x_2, ...)$ said to be scale invariant if if it satisfies \begin{equation}
   f(\lambda_1x_1, \lambda_2x_2, ...)=C(\lambda_1,\lambda_2, ...)f(x_1, x_2, ...).
\end{equation}


Since, RW and WR are dependent on the predictions of primary and alternatives i.e $P$ and $Q$,  change in size of the test samples will imply change in RW and WR. Let the changes in RW and WR be obtained as factors of $\lambda_1$ and $\lambda_2$. Thus, $\sigma_c (x_1, x_2)$ will be scale-invariant if  power-low dependency can be shown considering the following

\begin{equation}
    \sigma_c(\lambda_1x_1, \lambda_2x_2)=\frac{\lambda_1x_1-\lambda_2x_2}{\lambda_1x_1+\lambda_2x_2}
\end{equation}
where, $\lambda_1$ and $\lambda_2$ are the factors for RW and WR resulted in due to change in size of the test samples. 
First, taking the logarithm of both sides yields
\begin{equation}
    \ln \sigma_c(\lambda_1x_1, \lambda_2x_2) =\ln \left(\frac{\lambda_1x_1-\lambda_2x_2}{\lambda_1x_1+\lambda_2x_2}\right).
\end{equation}

Introducing a new function $F(x)$ defined as $F(\ln x)=\sigma_c(x)$ to above equation gives
\begin{align*}
    &\ln F(\ln \lambda_1+ \ln x_1, \ln \lambda_2+ \ln x_2) = \ln \left(\frac{\lambda_1x_1-\lambda_2x_2}{\lambda_1x_1+\lambda_2x_2}\right)\\
    &= \ln (\lambda_1x_1-\lambda_2x_2) - \ln (\lambda_1x_1+\lambda_2x_2)\\
    &= \ln \lambda_1+ \ln x_1- \ln\lambda_2- \ln x_2 - \ln \lambda_1- \ln x_1-\ln \lambda_2 -\ln x_2\\    
    &=\ln \lambda_1 - \ln\lambda_2 - \ln \lambda_1-\ln \lambda_2 + \ln x_1 - \ln x_2 - \ln x_1 -\ln x_2\\
    &= \ln (\lambda_1 - \lambda_2) - \ln (\lambda_1+ \lambda_2) + \ln (x_1 - x_2) - \ln (x_1 +x_2)\\
    &=\ln \left(\frac{\lambda_1 - \lambda_2}{\lambda_1 + \lambda_2}\right)+\ln \left(\frac{x_1 -x_2}{x_1 +x_2}\right)\\
    &=\ln \left(\frac{\lambda_1 - \lambda_2}{\lambda_1 + \lambda_2}\right) \left(\frac{x_1 -x_2}{x_1 +x_2}\right).
\end{align*}

Finally, the equation becomes 
\begin{equation}
    \ln F(\ln \lambda_1+ \ln x_1, \ln \lambda_2+ \ln x_2) = \ln \left(\frac{\lambda_1 - \lambda_2}{\lambda_1 + \lambda_2}\right) \left(\frac{x_1 -x_2}{x_1 +x_2}\right).
\end{equation}

Now, taking inverse function both sides yield

\begin{equation}
    e^{ F(\ln \lambda_1+ \ln x_1, \ln \lambda_2+ \ln x_2)} = e^{ \left(\frac{\lambda_1 - \lambda_2}{\lambda_1 + \lambda_2}\right) \left(\frac{x_1 -x_2}{x_1 +x_2}\right)}.
\end{equation}

Applying logarithm on both sides to above equation gives

\begin{equation}
    F(\ln \lambda_1+ \ln x_1, \ln \lambda_2+ \ln x_2) =  \left(\frac{\lambda_1 - \lambda_2}{\lambda_1 + \lambda_2}\right) \left(\frac{x_1 -x_2}{x_1 +x_2}\right).
\end{equation}

Changing the function back to $\sigma_c(x)$ gives

\begin{equation}
    \sigma_c(\lambda_1x_1, \lambda_2x_2) =  \left(\frac{\lambda_1 - \lambda_2}{\lambda_1 + \lambda_2}\right) \left(\frac{x_1 -x_2}{x_1 +x_2}\right).
\end{equation}
Representing right side as functions yields 
\begin{equation}
    \sigma_c(\lambda_1x_1, \lambda_2x_2) =  C(\lambda_1,\lambda_2) \sigma_c(x_1, x_2)
\end{equation}
where, $C$ is a function of two variables $\lambda_1$ and $\lambda_2$. Hence, proved that the metric $\sigma_c$ is scale invariant. Similarly, measures $\alpha, \xi_c, \xi_e$ and $\phi_e$ can also be proven as scale invariant. 
\end{proof}

\begin{theorem}
\label{th:dinvariant}
    All five measures $\sigma_c, \alpha, \xi_c, \xi_e$ and $\phi_e$ are data invariant.
\end{theorem}
\begin{proof}
Since the measures do not directly depend on the number of samples in each class or sequence in which samples are considered, balance or unbalance dataset, it does not have any direct impact on the measures. Moreover, measure considers the elements of prayatul matrix, which are counts of same-wise comparative outcomes of two algorithms. Therefore, even if the dataset is unbalanced, it will not have any impact on the measures. For instance, the measure $\sigma_c$  has two elements RW and WR. The element RW will be influenced only when the outcomes of primary algorithm is right and alternative algorithm is wrong for the same instances of the dataset it does not matter which the instances belong or even all the instances may belong to single class. Likewise, element WR will be influenced only when the outcomes of primary algorithm is wrong and alternative algorithm is right for the same instances of the dataset. Same is the case for the elements BR and BW. Therefore, all five measures $\sigma_c$, $\alpha, \xi_c, \xi_e$ and $\phi_e$ are data invariant.
\end{proof}

\begin{theorem}
\label{th:monotonic}
    All five measures $\sigma_c, \alpha, \xi_c, \xi_e$ and $\phi_e$ are monotonic.
\end{theorem}
\begin{proof}
Consistent improvement of $\sigma_c$, $\alpha, \xi_c, \xi_e$ and $\phi_e$  in the context primary algorithm means non-decreasing changes in these measures. Consistent improvement can happen either when the primary algorithm is right or alternative algorithm is wrong in majority of the instances i.e. the prayatul matrix elements which involves abstraction $Right$ for primary is more or abstraction $Wrong$ for alternative is more. Likewise, consistent improvement can also happen when the prayatul matrix elements which involves abstraction $Wrong$ for primary is less or abstraction $Right$ for alternative is less. 

Now, considering the measure $\sigma_c$, where numerator is sum of a positive RW and negative WR. Thus, increment of  RW or decrement of WR implies consistent improvement of $\sigma_c$ and it will have non-decreasing values. Similarly, increment of BR or decrement of BW implies consistent improvement of $\alpha$ and the values will be non-decreasing as well. In the same way,  consistent improvement of $\xi_c, \xi_e$ and $\phi_e$ will happen if  BR and/or RW increases or  if WR and/or BW decreases. Therefore, $\sigma_c$, $\alpha, \xi_c, \xi_e$ and $\phi_e$, all are non-decreasing under monotonic consistent improvement.
\end{proof}

\begin{theorem}
\label{th:continuous}
    All five measures $\sigma_c, \alpha, \xi_c, \xi_e$ and $\phi_e$ are continuous.
\end{theorem}

\begin{proof}
For continuity of a measure, small changes in number of sample has to cause a smaller impact on it. Let us consider a new instance is being added to the test dataset i.e. number of samples has increased by just 1. Now, let us examine the impact of a single instance on each of the elements of the prayatul matrix and on all five measures. Clearly, by the definition of prayatul matrix, only one of the values among BR, RW, WR and BW will be increased by 1 for the newly added instance, while other three will remain same.

Considering, the measure $\sigma_c$ which involves only two elements of the prayatul matrix i.e. RW and WR. Thus, there has the possibility that both RW and WR may remain unchanged for the newly added instance. If any one of RW and WR increase by 1, we will have 

\begin{align*}
    \sigma_c (P,Q)=\frac{RW-WR-1}{RW+WR+1}
\end{align*}

\begin{equation}
    \sigma_c (P,Q)=
\begin{cases}
 \frac{RW-WR+1}{RW+WR+1} & \text{if RW increases} \\
 \\
 \frac{RW-WR-1}{RW+WR+1} & \text{if WR increases}
\end{cases}
\end{equation}

Replacing RW-WR and RW+WR by $a$ and $b$ the above becomes

\begin{equation}
\sigma_c (P,Q)=
\begin{cases}
 \frac{a-1}{b+1} & \text{if RW increases} \\
 \\
 \frac{a-1}{b+1} & \text{if WR increases}
\end{cases}
\end{equation}

By Theorem~\ref{th:range} we have $a\le b$ always, which implies $a+1\le b+1$ as well as $a-1< b+1$. If $a>>1$ and $b>>1$ then by properties of ratio,
\begin{equation}
\left\{
\begin{aligned}
     &\frac{a+1}{b+1}	\equiv \frac{a}{b} \\     
    &\frac{a-1}{b+1}	\equiv \frac{a}{b}
\end{aligned}\right.
 \end{equation}  

Same is true for additional $k$ number of instances if $a>>k$ and $b>>k$. Thus, small change in number of samples $\sigma_c$ will have minimal so it is continuous. Similarly, the other measures $\alpha, \xi_c, \xi_e$ and $\phi_e$  can also be proven as continuous.
\end{proof}


\subsection{Prayatul Matrix Generation}
Prayatul matrix generation process is quite simple and straight forward. A simple algorithm called \emph{$\mathcal{D}$-Matrix Algorithm} is designed for generating prayatul matrix as shown in the Algorithm~\ref{DmatrixAlgo}. The algorithm takes three inputs: ground truth $G$, outcome $P$ of the primary algorithm and outcome $Q$ of the alternative algorithm. The entries of prayatul matrix is computed based on the abstraction levels as specified above and finally the algorithm returns the prayatul matrix $D$. The time complexity of the algorithm is $\mathcal{O}(N)$, where $N$ is the number instances.  

\begin{algorithm}
\caption{$\mathcal{D}$-Matrix Algorithm}\label{DmatrixAlgo}
\DontPrintSemicolon
\KwIn{$G, P, Q$}
\KwOut{ Prayatul Matrix $(\mathcal{D})$}
\SetKwBlock{Begin}{procedure}{end procedure}
\Begin($\text{generatePrayatulMatrix} {(} G, P, Q{)}$)
{
  $\mathcal{D}_{ij} \leftarrow 0, \forall i,j \in [1,2]$\;
  $N \leftarrow $number of instances in $G$\;
   \For{$i=1$ to $N$}
  {
    \lIf{$p_i=g_i$ AND $q_i=g_i$}
    {
        $\mathcal{D}_{11} \leftarrow \mathcal{D}_{11} + 1$
    }
    \lElseIf{$p_i=g_i$ AND $q_i\neq g_i$}
    {
        $\mathcal{D}_{12} \leftarrow \mathcal{D}_{12} + 1$
    }
    \lElseIf{$p_i\neq g_i$ AND $q_i = g_i$}
    {
        $\mathcal{D}_{21} \leftarrow \mathcal{D}_{21} + 1$
    }
    \lElseIf{$p_i\neq g_i$ AND $q_i\neq g_i$}
    {
      $\mathcal{D}_{22} \leftarrow \mathcal{D}_{22} + 1$
    }  
  }\label{endfor}  
  \Return{$\mathcal{D}$}
}
\end{algorithm}

\section{Experimental Analysis}
\label{sec:expAnalysis}
\subsection{Experimental Setup}
\subsubsection{Datasets}
To analyze efficacy of the proposed prayatul matrix based method as well as performance measures, datasets for classification problem are considered from Scikit-learn package~\cite{pedregosa2011scikit}. Three datasets considered for classification are Make Moons (MM), Make Circles (MC), and Linearly Separable (LS). All of the three classification datasets contain 40 test instances. Make Moons dataset is generated with parameter values noise=0.3 and random\_state=0. Make Circles dataset is generated with parameter values noise=0.2, factor=0.5 and random\_state=1. However, Linearly Separable dataset is generated using Make Classification base dataset with parameter values n\_features=2, n\_redundant=0, n\_informative=2, random\_state=1, and n\_clusters\_per\_class=1.



Four large-scale image datasets having hundreds of classes and containing huge numbers of instances are considered for analysis on deep learning models detailed as follows:
\begin{description}
    \item[\textbf{MNIST:}] \hspace*{5pt}The widely used MNIST~\cite{deng2012mnist} database (Modified National Institute of Standards and Technology database) of handwritten digits contains 70,000 samples and 10 classes.
    \item[\textbf{CIFAR10:}]  \hspace*{10pt}The CIFAR10 dataset~\cite{krizhevsky2009learning} contains of 60,000 samples of 32x32 colour images having 10 classes, with 6,000 images per class.
    \item[\textbf{cBirds:}] The caltech\_birds2010 (cBirds)  dataset~\cite{WelinderEtal2010} is an image dataset with photos of 200 bird species containing total number of categories of birds is 200 and there are 6,033 images in the 2010 dataset.
    \item[\textbf{eMNIST:}] \hspace*{7pt} The extended MNIST (eMNIST) dataset~\cite{cohen_afshar_tapson_schaik_2017} is a collection of handwritten digits and characters. There are a total of 62 classes and are unbalanced by class~\cite{pavan2021multi} having total 8,14,255 instances.   
\end{description}

\begin{figure*}[!h]
    \centering
    \includegraphics[width=1.5\columnwidth]{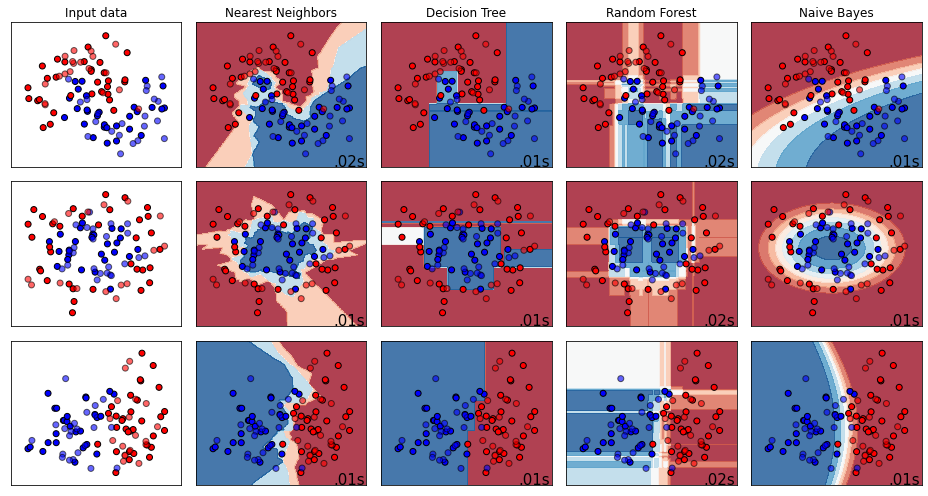}
    \caption{Results with classification algorithms}
    \label{fig:classResult}
\end{figure*}

\begin{table*}
    \caption{Proposed measures obtained for Nearest Neighbor in comparison to three other classifiers}
    \vspace{0.15cm}
    \label{tab:compMeasureClass}
    \centering
    
    \noindent\begin{tabular}{|@{}c|@{}c|@{}c|@{}c|@{}c|@{}c|@{}c|@{}c|@{}c|@{}c|@{}c|@{}c|@{}c|@{}c|@{}c|@{}c|}
    \hline
     Algos& \multicolumn{5}{c|}{Decision Tree}& \multicolumn{5}{c|}{Random Forest}& \multicolumn{5}{c|}{Naive Bayes}\\
       \hline
        Data&$\sigma_c$&$\alpha$&$\xi_c$&$\xi_e$&$\phi_c$&$\sigma_c$&$\alpha$&$\xi_c$&$\xi_e$&$\phi_c$&$\sigma_c$&$\alpha$&$\xi_c$&$\xi_e$&$\phi_c$\\
     \hline
        MM &1&0.95&1&1&0.98&1&0.95&1&1&0.98&1&0.95&1&1&0.98\\

        MC &0.56&0.9&0.95&0.9&0.88&0.5&0.9&0.95&0.9&0.88&0.69&0.9&0.95&0.89&0.88\\

        LS &-1&0.875&0.97&0.95&0.9&0&0.875&0.97&0.95&0.9&-1&0.875&0.97&0.95&0.9\\

     \hline    
    \end{tabular}

\end{table*}

\begin{table*}
    \caption{Accuracy, Precision and Recall based on confusion matrix for classification algorithms}
    \vspace{0.15cm}
    \label{tab:classNMeasure}
    \centering
    \begin{tabular}{|c|c|c|c|c|c|c|c|c|c|c|c|c|}
    \hline
     Algos& \multicolumn{3}{c|}{Nearest Neighbor}& \multicolumn{3}{c|}{Decision Tree}& \multicolumn{3}{c|}{Random Forest}& \multicolumn{3}{c|}{Naive Bayes}\\
       \hline
        Data&$Acc$&$Pre$&$Rec$&$Acc$&$Pre$&$Rec$&$Acc$&$Pre$&$Rec$&$Acc$&$Pre$&$Rec$\\
     \hline
        MM &0.98&0.95&1&0.95&0.91&1&0.93&0.95&0.9&0.88&0.9&0.86 \\
        MC &0.93&0.96&0.92&0.78&0.86&0.75&0.9&0.95&0.88&0.7&1&0.5 \\
        LS &0.93&0.95&0.91&0.95&0.95&0.95&0.93&1&0.86&0.95&1&0.91 \\
     \hline    
    \end{tabular}

\end{table*}

\subsubsection{ML/DL Models}

Supervised machine learning algorithms are considered from Scikit-learn package~\cite{pedregosa2011scikit} to analyse the efficacy of proposed evaluation method. Four widely used classification algorithms namely K-Nearest Neighbor (Nearest Neighbor), Decision Tree, Random Forest and Naive Bayes are considered. Specific parameters related to  classification algorithms are set as follows. For Nearest Neighbor $k$=3, for Decision Tree max\_depth=5, and for Random Forest max\_depth=5, n\_estimators=10, max\_features=1 are considered.
Further more, four widely used deep learning models are considered, which include ResNet50V2~\cite{he2016identity}, MobileNetV2~\cite{sandler2018mobilenetv2}, EfficientNet~\cite{tan2019efficientnet},  and  XceptionNet~\cite{chollet2017xception}.

\subsubsection{Implementation Details and System Configuration}
All implementations and executions are done under Jupiter Notebook server 6.4.10 environment with Python 3.10.4. The D-Matrix Algorithm for generating prayatul matrix and proposed performances measures are implemented in Python language ~\footnote{ Source codes of prayatul matrix and five scores are released through GitHub under GPLv3 License https://github.com/anupambis/Prayatul-for-classification
}. As mentioned above, widely used classification algorithms that are already implemented and openly available in Scikit-learn package~\cite{pedregosa2011scikit} are considered. Publicly available source code for  classification~\cite{classifierScikit} is considered as reference to setup the experimental environment. Since number of samples $n$ are less in case of datasets used for classification algorithms, leave-one-out cross-validation is performed by repeating the process $n$ times. However for large-scale image datasets used in deep learning models, 10-fold cross validation is performed and reported mean scores. All the experiments are done on the Computer having Intel(R) Core(TM) i7-8565U CPU @ 1.80GHz with 8  Cores, 	4.6GHz Speed, NVIDIA GeForce MX130 Graphics card, 16 GB RAM, 1TB HDD and 64-bit (AMD) Windows 10  Operating System. However, deep learning related experiments were done on Google Colab Pro+.

\subsection{Result Analysis}

The proposed direct comparison measures obtained with classification and deep learning models are analyzed from the perspective measure values as well as instance level comparison outcomes though prayatul matrix. Proposed measures are compared with the indications of confusion matrix based measures and then re-verified with instance level comparison entries in prayatul matrix.  

\subsubsection{Measure Value-based Analysis}
 Since, the proposed direct comparison approach pairs two algorithms for computing performance measure values, the interpretation of these measures are also to be done pairwise. The measure values presented in Table~\ref{tab:compMeasureClass} indicate that the performance of Nearest Neighbor algorithm is better in comparison to other three classification algorithms as mostly all measure values are positive.  Similar indication is given by the accuracy, precision and recall values as shown in Table~\ref{tab:classNMeasure}. However, these values certainly cannot tell us that the Nearest Neighbor is incapable of taking right decisions on certain instances of Linearly Separable dataset, while Decision Tree and Naive Bayes are capable of taking right decisions on those cases as indicated by negative comparative deviation values. Likewise, on both Make Moons and Make Circles datasets also the low comparative deviation values indicate that Nearest Neighbor is incapable of taking right decision on certain instances despite its superiority, while others can do.  The high polarization values indicate that Nearest Neighbor is polarized towards right decision in comparison to Decision Tree, Random Forest and Naive Bayes. The other three measures i.e. comparative rightness, effective rightness and effective superiority are also indicating superiority of Nearest Neighbor over Decision Tree, Random Forest and Naive Bayes on all the datasets. All these indications are clearly visible in the results presented in Fig.~\ref{fig:classResult}.

\begin{figure*}
     \centering
     \begin{subfigure}[b]{0.49\textwidth}
         \centering
          \begin{tikzpicture}
 
\begin{axis}  
[  
    width=\textwidth,
    height=1.15\textwidth,
    xbar, 
     enlargelimits=0.25,
     legend style={at={(1.12,1.05)}, 
    enlarge y limits=0.15, rotate=90,
    enlarge x limits=0.025, 
      anchor=south, legend columns=-1},    
    clip=false,    
    xlabel={GPU Time (minutes)}, 
    symbolic y coords={MNIST, CIFAR10, cBirds, eMNIST},  
    ytick=data,  
    xmin=0,
    xmax=1000,
    y tick label style={rotate=90,anchor=south},
    nodes near coords,  
    nodes near coords align={vertical},   
    nodes near coords style={anchor=west,inner ysep=1pt}
    ]  
   \addplot [draw = black,
    semithick,
    pattern = crosshatch,
    pattern color = red
] coordinates {(39.49,MNIST) (30.78,CIFAR10) (64,cBirds) (780.56,eMNIST)};

    \addplot [draw = black,
    semithick,
    pattern = crosshatch,
    pattern color = blue
] coordinates {(16.63,MNIST) (15.81,CIFAR10) (33.6,cBirds) (441.81,eMNIST)};  
    
    \addplot [draw = black,
    semithick,
    pattern = crosshatch,
    pattern color = green
] coordinates {(27.98,MNIST) (24.41,CIFAR10) (51.55,cBirds) (749.62,eMNIST)};  
    
    \addplot [draw = black,
    semithick,
    pattern = crosshatch,
    pattern color = brown
] coordinates {(31.61,MNIST) (30.14,CIFAR10) (87.25,cBirds) (665.51,eMNIST)}; 

\legend{ResNet, MobileNet, EfficientNet, XceptionNet}  
 
\end{axis}
 
\end{tikzpicture}
         \caption{Training Time}
         \label{fig:y equals x}
     \end{subfigure}    
     \begin{subfigure}[b]{0.49\textwidth}
         \centering
         \begin{tikzpicture}
 
\begin{axis}  
[  
    width=\textwidth,
    height=1.15\textwidth,
    xbar, 
     enlargelimits=0.25,
    legend style={at={(.77,0.27)}, 
    enlarge y limits=0.15,
    enlarge x limits=0.025, 
      anchor=north},    
    clip=false,    
    xlabel={Accuracy}, 
    symbolic y coords={MNIST, CIFAR10, cBirds, eMNIST},  
    ytick=data,  
    xmin=0,
    xmax=1,
    y tick label style={rotate=90,anchor=south},
    ]  
   \addplot [draw = black,
    semithick,
    pattern = crosshatch,
    pattern color = red
] coordinates {(0.9923,MNIST) (0.7200,CIFAR10) (0.9977,cBirds) (0.9852,eMNIST)};

    \addplot [draw = black,
    semithick,
    pattern = crosshatch,
    pattern color = blue
] coordinates {(0.9923,MNIST) (0.6860,CIFAR10) (0.9973,cBirds) (0.9036,eMNIST)};  
    
    \addplot [draw = black,
    semithick,
    pattern = crosshatch,
    pattern color = green
] coordinates {(0.9901,MNIST) (0.6317,CIFAR10) (0.9723,cBirds) (0.9076,eMNIST)};  
    
    \addplot [draw = black,
    semithick,
    pattern = crosshatch,
    pattern color = brown
] coordinates {(0.9933,MNIST) (0.7622,CIFAR10) (0.9980,cBirds) (0.9852,eMNIST)}; 

 
\end{axis}

\end{tikzpicture}
         \caption{Training Accuracy}
         \label{fig:three sin x}
     \end{subfigure}
        \caption{Training Accuracy and GPU Time (in min) required for training deep learning models}
        \label{fig:DLTrainTime}
\end{figure*}
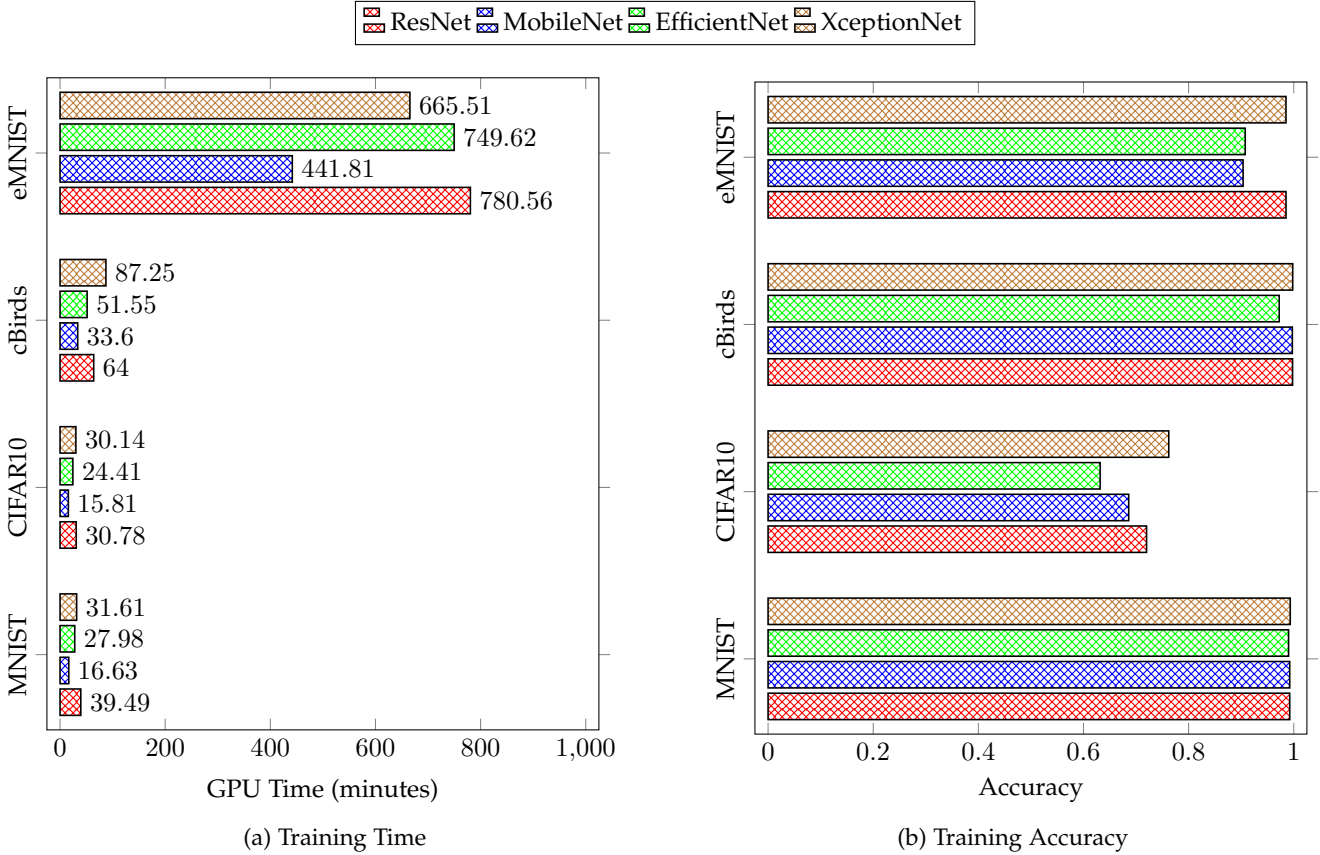

\begin{table*}[t]
    \caption{Proposed measures obtained for ResNet in comparison to three other deep learning models}
    \vspace{0.15cm}
    \label{tab:compMeasureDL}
    \centering
    \resizebox{\textwidth}{!}{  
    \noindent\begin{tabular}{|c|c|c|c|c|c|c|c|c|c|c|c|c|c|c|c|}
    \hline
     Models& \multicolumn{5}{c|}{MobileNet}& \multicolumn{5}{c|}{EfficientNet}& \multicolumn{5}{c|}{XceptionNet}\\
       \hline
        Data&$\sigma_c$&$\alpha$&$\xi_c$&$\xi_e$&$\phi_c$&$\sigma_c$&$\alpha$&$\xi_c$&$\xi_e$&$\phi_c$&$\sigma_c$&$\alpha$&$\xi_c$&$\xi_e$&$\phi_c$\\
     \hline
 MNIST&-0.0299&0.9910&0.9983&0.9966&0.9957&-0.0259&0.9899&0.9986&0.9972&0.9960&-0.2510&0.9904&0.9983&0.9967&0.9958\\
        CIFAR10&0.1712&0.5268&0.9380&0.8761&0.8295&-0.2562&0.6013&0.9384&0.8767&0.8296&-0.4892&0.5520&0.9263&0.8527&0.8179\\

        cBirds&-0.0288&-0.7982&0.4856&-0.0288&0.0012&0.1250&-0.7755&0.6792&0.4583&0.00381&-0.0224&-0.6904&0.5814&0.2628&0.0027\\   
        eMNIST& -0.0757&0.7539&0.9300&0.8601&0.7816&-0.2008&0.7543&0.9297&0.8593&0.7812&-0.0241&0.7521&0.9318&0.8637&0.7833\\ 
     \hline    
    \end{tabular}
}
\end{table*}

\begin{table*}[t]
    \caption{Accuracy, Precision and Recall based on confusion matrix for deep learning models}
    \vspace{0.15cm}
    \label{tab:confMeasureDL}
    \centering
    \resizebox{\textwidth}{!}{  
    \begin{tabular}{|c|c|c|c|c|c|c|c|c|c|c|c|c|}
    \hline
     Models& \multicolumn{3}{c|}{ResNet}& \multicolumn{3}{c|}{MobileNet}& \multicolumn{3}{c|}{EfficientNet}& \multicolumn{3}{c|}{XceptionNet}\\
       \hline
        Data&$Acc$&$Pre$&$Rec$&$Acc$&$Pre$&$Rec$&$Acc$&$Pre$&$Rec$&$Acc$&$Pre$&$Rec$\\
     \hline
        MNIST&0.9974&0.9974&0.9974&0.9974&0.9974&0.9974&0.9965&0.9965&0.9965&0.9977&0.9977&0.9977\\
        CIFAR10&0.8839&0.8876&0.8839&0.8657&0.8667&0.8657&0.8715&0.8713&0.8715&0.9149&0.9182&0.9149\\
        cBirds&0.0055&0.0055&0.0058&0.0043&0.0043&0.0049& 0.0038&0.0038&0.0037&0.0049&0.00497&0.0049\\
        eMNIST&0.8452&0.7526&0.7382&0.8541&0.7680&0.7489&0.8665&0.7789&0.7567&0.8481&0.7481&0.7459\\
        
     \hline    
    \end{tabular}
}
\end{table*}


The proposed measure values obtained for deep learning models on four large scale image datasets are presented in Table~\ref{tab:compMeasureDL}. Clearly, ResNet deviates from the decisions of MobileNet, EfficientNet and XceptionNet and those decisions are wrong as indicated by negative comparative deviation values in all four datasets.  The confusion matrix based measures presented in Table~\ref{tab:confMeasureDL} certainly cannot tell that ResNet takes wrong decision on certain instances while other takes right decision on those instances. Nevertheless, the comparative rightness, effective rightness and effective superiority values indicate that ResNet as the best performing model in almost all four dataset and  confusion matrix based measures also indicate the same. Even on highly imbalanced dataset eMNIST also  ResNet seems to be performing comparatively better than other three models as indicated comparative rightness, effective rightness and effective superiority. While confusion matrix based measures indicate all four models has almost same performance on eMNIST.  Surprisingly, all four models seems to perform worst in cBirds dataset, though training accuracy is high for all models as shown in Fig.~\ref{fig:DLTrainTime}. High negative polarity values indicate that ResNet and other three comparing models are highly polarized towards wrong decision. Though, confusion matrix based measures indicate that all models are highly inaccurate and ResNet is comparatively better but cannot tell that if wrong decisions of ResNet taken into account and MobileNet is actually better than ResNet in cBirds dataset as indicated by negative effective rightness value.

\begin{figure*}[t]
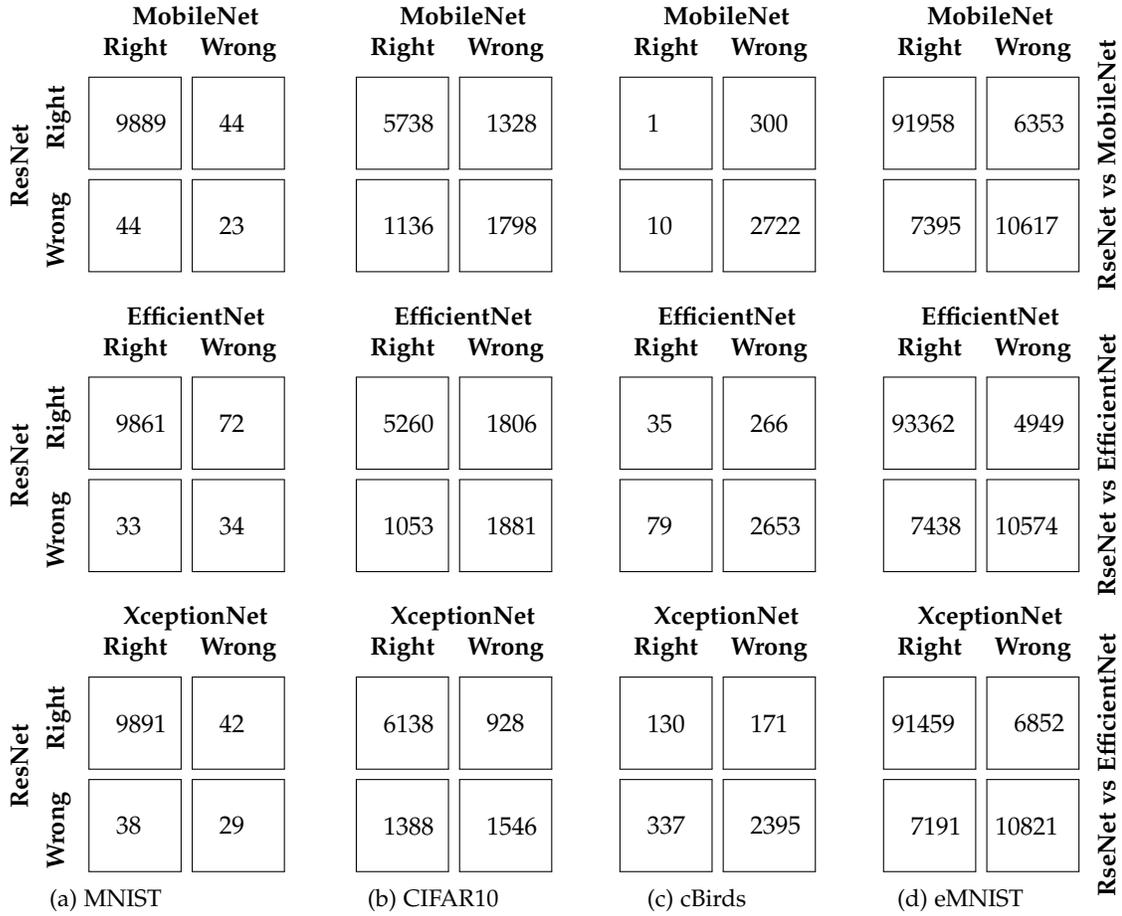

     \centering
     \begin{subfigure}[b]{0.3\columnwidth}
         \centering
   
 \noindent
\setlength\tabcolsep{0pt}
\begin{tabular}{c >{\bfseries \hspace{0.7em}}c @{\hspace{0.7em}}c @{\hspace{0.4em}}c @{\hspace{0.7em}}l}
   & & \multicolumn{2}{c}{\bfseries MobileNet }  \\
  & & \bfseries Right & \bfseries Wrong \\[0.5em]
  \multirow{4}{*}{\rotatebox{90}{\parbox{2.2cm}{\bfseries \centering ResNet }}}& \multirow{4}{*}{\rotatebox{90}{\parbox{2.5cm}{\bfseries \centering Wrong \hspace{10pt}   Right }}} & \MySCBox{9889} & \MySCBox{44} \\[2.6em]
  &  & \MySCBox{44} & \MySCBox{23}\\
  \\
  & & \multicolumn{2}{c}{\bfseries EfficientNet }  \\
  & & \bfseries Right & \bfseries Wrong \\[0.5em]
  \multirow{4}{*}{\rotatebox{90}{\parbox{2.2cm}{\bfseries \centering ResNet }}}& \multirow{4}{*}{\rotatebox{90}{\parbox{2.5cm}{\bfseries \centering Wrong \hspace{10pt}   Right }}} & \MySCBox{9861} & \MySCBox{72} \\[2.6em]
  &  & \MySCBox{33} & \MySCBox{34}\\
  \\
  & & \multicolumn{2}{c}{\bfseries XceptionNet }  \\
  & & \bfseries Right& \bfseries Wrong \\[0.5em]
  \multirow{4}{*}{\rotatebox{90}{\parbox{2.2cm}{\bfseries \centering ResNet }}}& \multirow{4}{*}{\rotatebox{90}{\parbox{2.5cm}{\bfseries \centering Wrong \hspace{10pt}   Right }}} & \MySCBox{9891} & \MySCBox{42} \\[2.6em]
  &  & \MySCBox{38} & \MySCBox{29}
\end{tabular}
    \caption{ MNIST}
     \end{subfigure}
     \hspace{1.5cm}
     \begin{subfigure}[b]{0.3\columnwidth}
         \centering
   
 \noindent
\setlength\tabcolsep{0pt}
\begin{tabular}{c >{\bfseries}r @{\hspace{0.7em}}c @{\hspace{0.4em}}c @{\hspace{0.7em}}l}
   & & \multicolumn{2}{c}{\bfseries MobileNet}  \\
  & & \bfseries Right & \bfseries Wrong \\[0.5em]
  &  & \MySCBox{5738} & \MySCBox{1328} \\[2.6em]
  &  & \MySCBox{1136} & \MySCBox{1798}\\
  \\
& & \multicolumn{2}{c}{\bfseries EfficientNet}  \\
  & & \bfseries Right & \bfseries Wrong \\[0.5em]
  &  & \MySCBox{5260} & \MySCBox{1806} \\[2.6em]
  &  & \MySCBox{1053} & \MySCBox{1881}\\
  \\
  & & \multicolumn{2}{c}{\bfseries XceptionNet}  \\
  & & \bfseries Right & \bfseries Wrong \\[0.5em]
  &  & \MySCBox{6138} & \MySCBox{928} \\[2.6em]
  &  & \MySCBox{1388} & \MySCBox{1546}
\end{tabular}
    \caption{ CIFAR10}
     \end{subfigure}
     \hspace{0.65cm}
     \begin{subfigure}[b]{0.3\columnwidth}
         \centering
   
 \noindent
\setlength\tabcolsep{0pt}
\begin{tabular}{c >{\bfseries}r @{\hspace{0.7em}}c @{\hspace{0.4em}}c @{\hspace{0.7em}}l}
   & & \multicolumn{2}{c}{\bfseries MobileNet }  \\
  & & \bfseries Right & \bfseries Wrong \\[0.5em]
 &  & \MySCBox{1} & \MySCBox{300} \\[2.6em]
  &  & \MySCBox{10} & \MySCBox{2722}\\
  \\
  & & \multicolumn{2}{c}{\bfseries EfficientNet}  \\
  & & \bfseries Right & \bfseries Wrong \\[0.5em]
  &  & \MySCBox{35} & \MySCBox{266} \\[2.6em]
  &  & \MySCBox{79} & \MySCBox{2653}\\
  \\
  & & \multicolumn{2}{c}{\bfseries XceptionNet}  \\
  & & \bfseries Right & \bfseries Wrong \\[0.5em]
  &  & \MySCBox{130} & \MySCBox{171} \\[2.6em]
  &  & \MySCBox{337} & \MySCBox{2395}
  \end{tabular}
    \caption{ cBirds}
     \end{subfigure}
      \hspace{0.65cm}
     \begin{subfigure}[b]{0.3\columnwidth}
         \centering
   
 \noindent
\setlength\tabcolsep{0pt}
\begin{tabular}{c >{\bfseries}r @{\hspace{0.7em}}c @{\hspace{0.4em}}c @{\hspace{0.7em}}c}
   & & \multicolumn{2}{c}{\bfseries MobileNet }  \\
  & & \bfseries Right & \bfseries Wrong &\multirow{5}{*}{\rotatebox{90}{\parbox{3.5cm}{\bfseries \centering RseNet vs MobileNet}}}\\[0.5em]
 &  & \MySCBoxB{91958} & \MySCBox{6353} & \\[2.6em]
  &  & \MySCBox{7395} & \MySCBoxB{10617}\\
  \\
  & & \multicolumn{2}{c}{\bfseries EfficientNet}  \\
  & & \bfseries Right & \bfseries Wrong & \multirow{5}{*}{\rotatebox{90}{\parbox{3.5cm}{\bfseries \centering RseNet vs EfficientNet}}}\\[0.5em]
  &  & \MySCBoxB{93362} & \MySCBox{4949} & \\[2.6em]
  &  & \MySCBox{7438} & \MySCBoxB{10574}\\
  \\
  & & \multicolumn{2}{c}{\bfseries XceptionNet}  \\
  & & \bfseries Right & \bfseries Wrong & \multirow{5}{*}{\rotatebox{90}{\parbox{3.5cm}{\bfseries \centering RseNet vs EfficientNet}}}\\[0.5em]
  &  & \MySCBoxB{91459} & \MySCBox{6852} & \\[2.6em]
  &  & \MySCBox{7191} & \MySCBoxB{10821}
\end{tabular}
    \caption{ eMNIST}
     \end{subfigure} 
     
     \hfill
    \caption{Prayatul Matrices of deep learning models}
    \label{fig:compMatDL}
\end{figure*}

\subsubsection{Instance Level Analysis}

Since the prayatul matrix elements are the counts of instance level comparison of outcomes of two ML models and the proposed measures are designed based on it, so the measure values give  instance level comparative performance of the two ML models. To reaffirm this, the indications noted above are analyzed with the prayatul matrices obtained for classification and deep learning models. Clearly, prayatul matrices for deep learning models presented in Fig.~\ref{fig:compMatDL} shows that ResNet mostly having more right decision count in comparison to others as RW count is more than WR count. Also, BR+RW is much higher than BW+WR. This reaffirms the indications in measure based analysis. Specifically, for cBirds dataset, BW is very high implying high negative polarization value. Also, it is clearly visible that high BR+RW values in MNIST dataset, which means ResNet mostly takes right decision and the same is reaffirmation of highly positive polarization value. Similarly, the prayatul matrices for classification algorithms presented in Fig.~\ref{fig:CompMatClass}, Nearest Neighbor mostly having more right decision count. However, as noted earlier, for Linearly Separable dataset WR count is more than RW count so it resulted negative comparative deviation. The different values like Tree Positive (TP), False Positive (FP), False Negative (FN), True Negative (TN) in the confusion matrices of DL models and classification algorithms certainly cannot give these important insights.

\begin{figure*}[!ht]
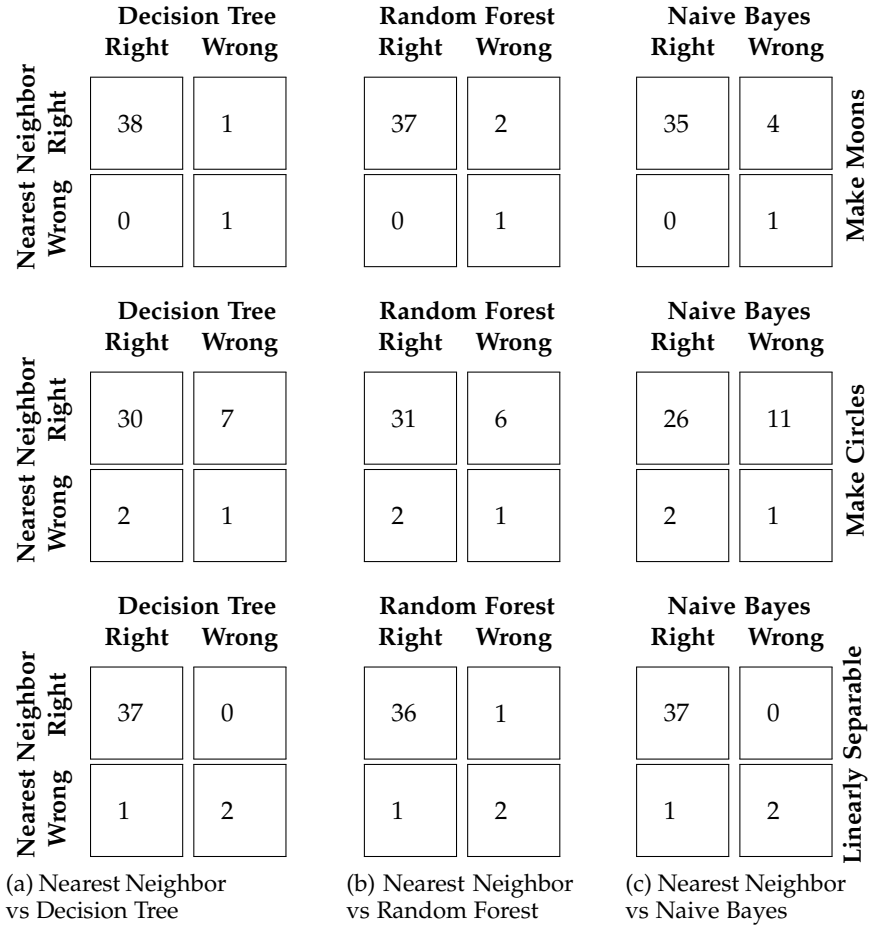

     \centering
   
     \begin{subfigure}[b]{0.33\columnwidth}
         \centering
   
 \noindent
\renewcommand\arraystretch{1}
\setlength\tabcolsep{0pt}
\begin{tabular}{c >{\hspace{0.25em}\bfseries}r @{\hspace{0.7em}}c @{\hspace{0.4em}}c @{\hspace{0.7em}}l}
   & & \multicolumn{2}{c}{\bfseries Decision Tree}  \\
  \multirow{5}{*}{\rotatebox{90}{\parbox{3.15cm}{\bfseries  Nearest Neighbor}}}& & \bfseries Right& \bfseries Wrong \\[0.5em]
  & \multirow{2}{*}{\rotatebox{90}{\hfill Right }} & \MySBox{38} & \MySBox{1} \\[2.4em]
  & \multirow{2}{*}{\rotatebox{90}{Wrong}} & \MySBox{0} & \MySBox{1}\\
  \\
   & & \multicolumn{2}{c}{\bfseries Decision Tree}  \\
  \multirow{5}{*}{\rotatebox{90}{\parbox{3.5cm}{\bfseries \centering Nearest Neighbor}}}& & \bfseries Right& \bfseries Wrong \\[0.5em]
  & \multirow{2}{*}{\rotatebox{90}{Right}} & \MySBox{30} & \MySBox{7} \\[2.4em]
  & \multirow{2}{*}{\rotatebox{90}{Wrong}} & \MySBox{2} & \MySBox{1}\\
  \\
     & & \multicolumn{2}{c}{\bfseries Decision Tree}  \\
  \multirow{5}{*}{\rotatebox{90}{\parbox{3.5cm}{\bfseries \centering Nearest Neighbor}}}& & \bfseries Right& \bfseries Wrong \\[0.5em]
  & \multirow{2}{*}{\rotatebox{90}{Right}} & \MySBox{37} & \MySBox{0} \\[2.4em]
  & \multirow{2}{*}{\rotatebox{90}{ Wrong}} & \MySBox{1} & \MySBox{2}

\end{tabular}
    \caption{Nearest Neighbor vs Decision Tree}
    
     \end{subfigure}\hspace{1.5cm}
     \begin{subfigure}[b]{0.34\columnwidth}
         \centering
   
 \noindent
\renewcommand\arraystretch{1}
\setlength\tabcolsep{0pt}
\begin{tabular}{c >{\bfseries}r @{\hspace{0.4em}}c @{\hspace{0.4em}}c @{\hspace{0.7em}}l}
   \multicolumn{3}{c}{\bfseries Random Forest}\\
   \bfseries Right &&\bfseries Wrong \\[0.5em]
   \MySBox{37} && \MySBox{2}\\[2.4em]
   \MySBox{0} && \MySBox{1} \\
 \\
    \multicolumn{3}{c}{\bfseries Random Forest}\\
   \bfseries Right &&\bfseries Wrong \\[0.5em]
   \MySBox{31} && \MySBox{6}\\[2.4em]
   \MySBox{2} && \MySBox{1} \\
  \\
    \multicolumn{3}{c}{\bfseries Random Forest}\\
   \bfseries Right &&\bfseries Wrong \\[0.5em]
   \MySBox{36} && \MySBox{1}\\[2.4em]
   \MySBox{1} && \MySBox{2} \\ 
\end{tabular}
    \caption{ Nearest Neighbor vs Random Forest}
 
     \end{subfigure}  \hspace{0.5cm}
     \begin{subfigure}[b]{0.33\columnwidth}
       \centering
   
 \noindent
\renewcommand\arraystretch{1}
\setlength\tabcolsep{0pt}
\begin{tabular}{c >{\bfseries}r @{\hspace{0.4em}}c @{\hspace{0.4em}}c @{\hspace{0.7em}}l}
    \multicolumn{3}{c}{\bfseries Naive Bayes}\\
   \bfseries Right &&\bfseries Wrong \\[0.5em]
   \MySBox{35} && \MySBox{4}& \multirow{4}{*}{\rotatebox{90}{\parbox{2.2cm}{\bfseries \centering Make Moons}}}\\[2.4em]
   \MySBox{0} && \MySBox{1}\\
   \\
   \multicolumn{3}{c}{\bfseries Naive Bayes}\\
   \bfseries Right &&\bfseries Wrong \\[0.5em]
   \MySBox{26} && \MySBox{11}& \multirow{4}{*}{\rotatebox{90}{\parbox{2.2cm}{\bfseries \centering Make Circles}}}\\[2.4em]
   \MySBox{2} && \MySBox{1} \\
   \\
    \multicolumn{3}{c}{\bfseries Naive Bayes}\\
   \bfseries Right &&\bfseries Wrong &\multirow{5}{*}{\rotatebox{90}{\parbox{3.5cm}{\bfseries \centering Linearly Separable}}} \\[0.5em]
   \MySBox{37} && \MySBox{0}& \\[2.4em]
   \MySBox{1} && \MySBox{2} \\
\end{tabular}

  \caption{ Nearest Neighbor vs Naive Bayes}
    
     \end{subfigure}

    \caption{Prayatul Matrices of Nearest Neighbor with three other classification algorithms}
 
    \label{fig:CompMatClass}
\end{figure*}

\section{Related Work}
\label{sec:relwork}
The origin of widely used confusion matrix for ML can be traced back to Kerl Pearson's work~\cite{pearson1904theory} in 1904, where he referred it as contingency table. Later on the term confusion matrix~\cite{townsend1971theoretical} was used in the context of psychology. While the confusion matrix used in ML to evaluate both classification~\cite{Zhang2017} and clustering algorithms~\cite{xu2005survey} is prepared on the basis of ground truth labels and predicted labels. The conventional way to compare performance of two ML models is to prepare two separate confusion matrices for each model and generate scores to compare. One of the major drawbacks of this approach is it lacks direct comparison of ML models on individual instances of the dataset. The proposed direct comparison approach prepares a single comparison matrix by comparing instance level outcomes of two ML models w.r.t. ground truth. Earlier Dietterich~\cite{dietterich1998approximate} demonstrated a similar idea to construct comparison matrix for misclassified instances two ML model as an application of McNemar's test~\cite{everitt1977r}. However, Dietterich's approach focuses simply on misclassified instances, and analysis is done under the null hypothesis \emph{the two algorithms should have the same error rate}. Whereas, proposed approach considers all instances i.e. both correctly classified and misclassified instances to prepare the comparison matrix. Moreover, analysis is done on the basis of scores not null hypothesis. In recent years, few attempts to prepare confusion matrix alternatively are done. A three-way confusion matrix is designed that visualizes the degree of algorithm confusion within different classes~\cite{xu2020three}. The construction of basic probability assignment (BPA) based on the confusion matrix has also been studied in the context of classification problem~\cite{deng2016improved}. Simplified confusion matrix visualization techniques are also designed for better visualization of classes~\cite{susmaga2004confusion,beauxis2014simplifying}. However, none of these approaches considered direct comparison of ML models at instance level or simply comparison at instance level.

        

\section{Conclusion}
\label{sec:conclusion}
In this paper, an alternative approach is proposed for direct comparison of supervised ML models at instance level of test datasets. A direct comparison matrix called \emph{Prayatul Matrix} is prepared to account instance level comparison of outcomes of two ML algorithms. Five measures are designed based on the elements of prayatul matrix, which include comparative deviation, polarization, comparative rightness, effective rightness and effective superiority. Results on both classification and deep learning models showed that these measures can give some important insight about the outcomes of algorithms by comparing those directly at instance level. Also, these measures are equally capable of indicating the right decisions of algorithms as conventional measures like accuracy, recall and precision. For instance, two ML algorithms having same accuracy doesn't mean that their outcomes are same at instance level, comparative deviation and polarization give indications on such differences.  While comparative rightness, effective rightness and effective superiority measures give the indication for right decisions of comparing algorithms. Moreover, interpretation of the measures is simple, the rule-of-thumbs for end-users is highly positive values imply good performance and negative implies bad performance of primary algorithm. The proposed direct comparison approach has certain limitations in the context of clustering as it requires ground truth and like confusion matrix based measures, it will require prior adjustment of clustering labels.

 \section*{Acknowledgements}
 
 This work is supported by the Science and Engineering Board (SERB), Department of Science and Technology (DST) of the Government of India under Grant No. ECR/2018/000204.



\ifCLASSOPTIONcaptionsoff
  \newpage
\fi



%

\bibliographystyle{IEEEtran}
\bibliography{myrefs}

%

\begin{IEEEbiography}[{\includegraphics[width=.9in,height=1.2in,clip,keepaspectratio]{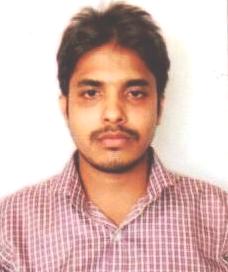}}]{Dr. Anupam Biswas}
  is currently working as an Assistant Professor in the Department of Computer Science and Engineering, National Institute of Technology Silchar, Assam, India. He has received Ph.D. degree in computer science and engineering from Indian Institute of Technology (BHU), Varanasi, India in 2017. He has received M. Tech. Degree in computer science and engineering from Motilal Nehru National Institute of Technology Allahabad, Prayagraj, India in 2013 and B. E.  degree in computer science and engineering from Jorhat Engineering College, Jorhat, Assam in 2011. He has published several research papers in transactions, reputed international journals, conference and book chapters. His research interests include Machine learning, Social Networks, Computational music, Information retrieval, and Evolutionary computation. He has four granted patents, out of with three are Germany patents and one South African patent. He is the Principal Investigator of two on-going DST-SERB sponsored research projects in the domain of machine learning and evolutionary computation. He has served as Program Chair of International Conference on Big Data, Machine Learning and Applications (BigDML 2019) and Publicity Chair of BigDML 2021. He has served as General Chair of 25th International Symposium Frontiers of Research in Speech and Music (FRSM 2020) and co-edited the proceedings of FRSM 2020 published as book volume in Springer AISC Series. He has edited five books that are published by various series of Springer. Also edited a book with Advances in Computers book Series of Elsevier.
\end{IEEEbiography}





\vfill

\end{document}